\documentclass[12pt, reqno]{amsart}

\usepackage{amsmath, amssymb, amsfonts, amsbsy, amsthm, latexsym}
\usepackage{graphicx}
\usepackage{verbatim}
\usepackage{color}
\usepackage{algorithm}
\usepackage{algorithmic}
\newtheorem{theorem}{Theorem}[section]

\newtheorem{lemma}[theorem]{Lemma}
\newtheorem{corollary}[theorem]{Corollary}

\theoremstyle{definition}

\newtheorem{example}[theorem]{Example}

\numberwithin{equation}{section}

\numberwithin{equation}{section}

\usepackage{enumerate}

\def\0{{\bar{0}}}

\def\R{{\mathbb{R}}}
\def\Z{{\mathbb{Z}}}

\begin{document}

\title[]{Stochastic Markov gradient descent and training low-bit neural networks}

 \author{Jonathan Ashbrock}
\address{Department of Mathematics, Vanderbilt University, Nashville, TN 37240, USA}
\email{jonathan.ashbrock@vanderbilt.edu}
%%%\thanks{J. Ashbrock was supported in part by...}

\author{Alexander M. Powell}
\address{Department of Mathematics, Vanderbilt University, Nashville, TN 37240, USA}
\email{alexander.m.powell@vanderbilt.edu}
%%%\thanks{A.M. Powell was supported in part by NSF DMS Grant ?????, and gratefully acknowledges the Academia Sinica Institute of Mathematics (Taipei, Taiwan) for their hospitality and support.}

%    General info

%\date{January 1, 2001 and, in revised form, June 22, 2001.}
%\date{\today}

\

\begin{abstract}
The massive size of modern neural networks has motivated substantial recent interest in neural network quantization. We introduce Stochastic Markov Gradient Descent (SMGD), a discrete optimization method applicable to training quantized neural networks.  
The SMGD algorithm is designed for settings where memory is highly constrained during training.
We provide theoretical guarantees of algorithm performance as well as encouraging numerical results. 
%In particular, we exhibit that our algorithm is state-of-the-art in the setting where memory is highly constrained during training.
%The algorithm is
\end{abstract}

\maketitle

\section{Introduction} \label{intro-sec}

Neural networks are a widely used tool for classification and regression tasks, \cite{Lecun,AlexNet}.  Given training data $\{x_j\}_{j=1}^m \subset \R^d$ and a set of labels $\{l_j\}_{j=1}^m \subset \R$,
the general goal is to learn a function $y$ that explains the training set by $$y(x_j) \approx l_j.$$
Neural networks address this by using a specially structured output function $y(x)=y(x,w)$ that is parametrized by a high-dimensional vector $w\in \R^n$ of weights and biases.
%
%\vspace{2in}
In a standard feedforward neural network, $y$
is an iterated composition of nonlinear activations and affine maps \cite{NN_Book}.  
More generally, when the training data consists of objects with particular structure, such as images or time series, the output function $y$
may incorporate additional components such as convolutional neurons \cite{Lecun} or feedback \cite{LSTM}.

The universal approximation theorem \cite{Cybenko} and later advances, e.g., \cite{Daubechies2019, zouweishen, Zhou2, Zhou1}, provide a theoretical foundation for neural networks,
and show that weight parameters $w$ can be selected so that the network output $y$ expresses a wide class of input-output relationships.
While neural networks enjoy approximation-theoretic power, 
the large size of the network weight set $w$ creates nontrivial practical challenges during implementation:
\begin{itemize}
\item Nonconvexity of the cost function leads to non-unique minima during training.
\item Slow training times can occur due to the large number of network parameters.
\item Large networks yield slow signal propagation and consequently slow classification.
\item Large amounts of memory are needed to store the network parameters.
\end{itemize}
These computational burdens have motivated the study of {\em quantized neural networks}.  In a standard neural network, the weight parameters $w$ are full-precision floating point numbers.
Instead, quantized neural networks use weight parameters that are intentionally represented using only a small number of bits.  For example, in the extreme case of
 binary neural networks, each weight only contains a single bit of information and so is constrained to take one of only two possible values.  

It has been shown recently, \cite{Courbariaux,DoReFa,xnor}, that quantized neural networks can 
match the state-of-the-art performance obtained by comparably-sized full-precision neural networks.  
Somewhat paradoxically, overparametrization creates computational challenges for implementing neural networks, but it also provides flexibility which
allows heavily quantized, even one-bit, networks to perform well.  Moreover, the use of low-bit neural networks reduces the memory requirements
needed to store network parameters, can be used to speed up signal propagation through networks \cite{xnor}, and can also be viewed as having a regularizing effect.

In this work, we address two aspects of the neural network quantization program.
First, we introduce a method, stochastic Markov gradient descent (SMGD), that produces neural networks that are low-memory during training as well as at run-time.
For comparison, in \cite{JMLRPaper} network weights are quantized at run-time but the method requires storage of full-precision auxiliary weights 
for the parameter update step during training which results in {increased} train-time memory requirements.
Later works, \cite{DoReFa}, produce  smaller run-time memory requirements by quantizing gradients and activations. Moreover, this has the effect of faster training because quantized gradients and activations allow access to bitwise operations during both the forward and the backward pass.  We place particular emphasis on the memory requirements during the training phase since existing quantization methods typically require increased memory requirements during network learning. Our method is the first to our knowledge that allows training of highly-accurate networks while memory is constrained at both train and run-time.

Secondly, the theoretical understanding of 
quantized neural networks 
is still being developed.
The problem of neural network quantization forces one to solve a discrete optimization problem in extremely high dimensions rather than a continuous problem. 
This high dimensionality disallows the use of many standard discrete optimization techniques.
Therefore, existing methods often involve an ad hoc blend of gradient-based methods and discrete optimization techniques.
For example, \cite{Kmeans} uses k-means to cluster similar weights together before quantization. The algorithm in \cite{Courbariaux} quantizes weights during the forward pass while applying the gradient descent update to full-precision, pre-quantized weights. More recent methods \cite{DoReFa,JMLRPaper} generally involve a mild variation on this last idea to achieve goals including quantized gradients, activations, or to apply these ideas to recurrent neural networks.  The work in \cite{Penghang} quantizes neural networks using an approach based
on blended coarse gradient descent.
Towards a more theoretically robust understanding, the work in \cite{LossAware} incorporates quantization error directly into the cost function. 
Our approach is based on a simple probabilistic variation of stochastic gradient descent, and proves theoretical performance guarantees which are highly coincident with their counterparts in traditional stochastic gradient descent. These results give us intuition for how the networks learn and yield evidence for the effectiveness of stochastic Markov gradient descent as a tool for quantizing neural networks. 

The main contributions of this paper are:
\begin{itemize}
\item We introduce stochastic Markov gradient descent (SMGD) for producing neural networks whose weights are fully quantized during {both} training and at run-time, allowing one to learn accurate networks in low-memory environments, see Section \ref{smgd:sec}.
\item We prove theoretical performance guarantees for SMGD in a general setting and draw strong comparisons to comparable results for stochastic gradient descent, see Theorems \ref{first_stoch_theorem} and \ref{iterate_convergence_thm}.
\item We numerically validate the SMGD algorithm and show that it performs well on various benchmark sets for image classification, see Section \ref{numerical_section}
\item We highlight the setting where memory is constrained during training, and show there are instances where networks trained by SMGD can outperform full-precision networks in a bit-for-bit comparison. 
\end{itemize}
The remainder of the paper is organized as follows. Section \ref{background:sec} covers some necessary background information on gradient descent and neural network training. Section \ref{smgd:sec} introduces the SMGD algorithm and gives a brief discussion of intuition for why the algorithm works. 
Section \ref{convergence_section} proves our first main results on the behavior of the cost function $f(x^t)$ under iterations of SMGD, see Theorem \ref{first_stoch_theorem}.
Section \ref{stoch_conv_sec} proves our next main results on rates of convergence for the iterates $x^t$ of SMGD in the special case of strongly convex cost functions $f$, see Theorem \ref{iterate_convergence_thm}.  
Section \ref{non_stoch_sec} collects several corollaries of our main results to illustrate the performance of SMGD in the non-stochastic setting, i.e., when we have access to the gradient itself.
Section \ref{numerical_section} contains numerical results which show that SMGD performs well in various settings.

\section{Background: stochastic gradient descent} \label{background:sec}

Neural network training is the process of using labelled training data $\{(x_j, l_j)\}_{j=1}^m$ to determine a good choice of network parameters $w$.  
Training is typically formulated as a minimization problem
\begin{equation} \label{f-min-prob}
\min_{w \in \R^n} f(w),
\end{equation}
where $f:\R^n \to \R$ is a cost function associated to the network and training data.
In machine learning in particular, $f$ is often of the form
\begin{equation} \label{f-sum-eq}
f(w) = \frac{1}{m} \sum_{i=1}^m f^i(w),
\end{equation}
where $f^i(w)$ measures an error between the label $l_i$ and the network output $y(x_i, w)$ for the $i${th} piece of training data.

The backpropogation algorithm allows efficient computation of $\nabla f^i(w)$, a portion of the gradient of our cost function, which in turn opens the toolbox of gradient-based methods for model selection.
Standard gradient descent is an iterative method which addresses \eqref{f-min-prob} by making updates in the direction of the negative gradient.
However, because of the form \eqref{f-sum-eq}, even if $\nabla f^i$ may be efficiently computed, it may be slow to compute the entirety of $\nabla f$ if $m$ is relatively large.  In practice, $m$ is often extremely large as we have access to larger and larger data sets to learn from.

Given a differentiable function $f:\mathbb{R}^n\rightarrow \mathbb{R}$, we say that a stochastic function $G:\R^n \rightarrow \mathbb{R}$  is an {\em unbiased
	estimator} of $\nabla f$ if $\mathbb{E} [G(x)] = \nabla f(x)$ where the expectation is with respect to the realization of $G$.  
	In the case when $f$ is of the form \eqref{f-sum-eq}, typical examples of unbiased estimators $G$ are:
	\begin{itemize}
	\item {\em Uniform}.  Draw $i$ uniformly at random from $\{1, \cdots, m\}$ and let $G(x) = \nabla f^i(x)$.
	\item {\em Mini-batch estimates}.  Draw $k$ distinct integers $i_1, \cdots, i_{k}$ uniformly at random without replacement from $\{1, \cdots, m\}$, and let $G(x)= \frac{1}{k} \sum_{j=1}^k \nabla f^{i_j}(x)$.
	\end{itemize}

{\em Stochastic gradient descent} (SGD) addresses the minimization problem \eqref{f-min-prob} by updating the parameter vector $w^t$ at step $t$ with the following the iteration
\begin{equation} \label{SGD-def}
w^{t+1}=w^t- \lambda \ G^t(w^t),
\end{equation}
where $G^t(w^t)$ is an unbiased estimate of $\nabla f(w^t)$ at iterate $t$ of SGD.
We consider the case when the learning rate $\lambda \in (0,\infty)$ is constant, but it is also common to vary the learning at each iteration.
Convergence properties of stochastic gradient descent are well-studied, especially for machine learning, e.g., \cite{adagrad,adam,Convexity_Result,Needell}.
%A common assumption for theoretical analysis of stochastic gradient descent is that the function $f$ being minimized has a Lipschitz continuous gradient and satisfies some form of convexity, but there are also recent advances which operate under weaker conditions \cite{ }. 

\section{Stochastic Markov Gradient Descent} \label{smgd:sec}

	In this paper, our goal is to minimize a differentiable function $f:\mathbb{R}^n\rightarrow \mathbb{R}$ constrained to a scaled lattice $\alpha \mathbb{Z}^n$ given access to unbiased estimators of the gradient $\nabla f$.  Our approach, stochastic Markov gradient descent (SMGD), generalizes the least squares Markov gradient descent algorithm that was introduced for digital halftoning in \cite{Shen}, and is a variant of SGD where additional randomness is employed to allow the iterates to remain on the lattice. Throughout the remainder of this work, we let $G^t$ denote an unbiased estimator of the gradient at step $t$. We generally use subscripts to denote coordinates of vectors, so that $x^t_i$ denotes the $i$th coordinate of $x^t \in \R^n$ and $G^t(x^t)_i$ denotes the $i${th} coordinate of the unbiased estimator $G^t(x^t)$ of $\nabla f(x^t)$.
	
	The stochastic Markov gradient descent algorithm is described below.

	\begin{algorithm}[h] \label{smgd-alg}
		\textbf{Stochastic Markov Gradient Descent (SMGD)}
		
		Input: $f:\mathbb{R}^n\rightarrow \mathbb{R}$, stepsize $\alpha$, initial $x^0 \in \alpha \mathbb{Z}^n$,
		number of iterations $T$, normalizer $\eta>0$ 
		
		Output: $x^{T}\in \alpha \mathbb{Z}^n$, an estimate of the minimizer
		\begin{algorithmic}
			\FOR{$t=1,\dots, T$ iterations}	
			\STATE Compute an unbiased estimator $G^t(x^t)$ of the gradient vector $\nabla f(x^t)$
			\FOR{each coordinate $x^t_i$}
			\STATE Let $\Delta^t_i$ be a Bernoulli random variable with $\mathbb{P}[\Delta^t_i=1]=\min(\left| G^t(x^t)_i \right|/\eta,1)$\\
			Update $x_i^{t+1}=x_i^t - \alpha \cdot sgn(G^t(x^t)_i)\Delta^t_i$
			\ENDFOR
			\ENDFOR
	\end{algorithmic}\end{algorithm}
	
	We shall make the following probabilistic assumptions for SMGD throughout the paper:
	\begin{itemize}
	\item We assume that $G$ is an unbiased estimator for $\nabla f$, and that $\{G^t\}_{t=1}^T$ are independent identically distributed versions of $G$.
		\item We assume that each unbiased estimator $G^t$ is independent of $x^t$, so that 
	\begin{equation} \label{cond-Gt-E-eq}
		\mathbb{E}[G^t(x^t) | x^t] = \nabla f(x^t).
	\end{equation} 
	Our analysis will require a slightly stronger independence assumption than \eqref{cond-Gt-E-eq}.  
	Let $\mathcal{E}^t$ denote the event $\| \nabla G^t(x^t)\|_\infty \leq \eta$. We assume further that
	\begin{equation} \label{strong-indep-assump}
	\mathbb{E}[G^t(x^t) | x^t, \mathcal{E}^t] = \nabla f(x^t).
	\end{equation}
	\item We assume that the conditional distribution of $\Delta^t_i$ given $G^t$ and $x^t$ is a Bernoulli distribution 
	with
	\begin{equation} \label{Delta-def}
	\mathbb{P}[\Delta^t_i=1 \thinspace | \thinspace G^t, x^t]=\min(\left| G^t(x^t)_i \right|/\eta,1).
	\end{equation}
	\end{itemize}

Standard stochastic gradient descent \eqref{SGD-def} makes updates that move non-discretely in the negative gradient direction $-\nabla f(x^t)$ in expectation.  However, SGD does not in general produce solutions $x^t$ that are constrained to the lattice $\alpha \mathbb{Z}^n$.  To remain constrained to the lattice $\alpha \mathbb{Z}^n$, one should only make discrete updates in each direction. 
Therefore, SMGD instead updates each coordinate of $x^t$ by a fixed amount \textit{with some probability} chosen so that the expected update remains in the same direction as SGD.
 To see this, note that if
  $\mathcal{E}^t$ is the event that $\| G^t(x^t)\|_\infty \leq \eta$ then, by \eqref{strong-indep-assump} and \eqref{Delta-def}, one has
 \begin{align}
 \mathbb{E}[x^{t+1}_i \thinspace | \thinspace x^t, \mathcal{E}^t]
 &=\mathbb{E}[x^t_i-\alpha\cdot\text{sgn} \left(G^t(x^t)_i\right)\Delta_i \thinspace | \thinspace x^t, \mathcal{E}^t]\notag\\
 &= \mathbb{E} \Big[ \mathbb{E}[x^t_i-\alpha\cdot\text{sgn} \left(G^t(x^t)_i\right)\Delta_i \thinspace | \thinspace x^t, G^t,\mathcal{E}^t] \thinspace \Big| \thinspace x^t, \mathcal{E}^t  \Big]\notag \\
 &= \mathbb{E} \Big[ x^t_i-\alpha\cdot\text{sgn} \left(G^t(x^t)_i\right) \frac{|G^t(x^t)_i|}{\eta} \thinspace \Big| \thinspace x^t,\mathcal{E}^t \Big]\notag \\
 &= x^t_i - \frac{\alpha}{\eta} \mathbb{E} \Big[ G^t(x^t)_i \thinspace \Big| \thinspace x^t,\mathcal{E}^t \Big]\notag \\
 &=x^t_i-\frac{\alpha}{\eta}\frac{\partial f}{\partial x_i}(x^t). \label{SMGD-ave}
 \end{align}
 
 In view of \eqref{SMGD-ave}, SMGD can be seen as a modification of SGD that keeps iterates $x^t$ on $\alpha \Z^n$ by introducing extra noise at each update step.
 For this reason, the majority of our error analysis will occur conditioned on the event $\mathcal{E}^t$ which led to the interpretation \eqref{SMGD-ave}.
There are similarities between the lattice resolution $\alpha$ in SMGD and the learning rate in standard SGD; 
we shall see that some convergence properties of SMGD rely on $\alpha$ in the same way that SGD relies on the learning rate, e.g., see Theorem \ref{iterate_convergence_thm}.

Stochastic Markov gradient descent  follows the nomenclature used for least squares Markov gradient descent in \cite{Shen}.  In particular, since the estimators $G^t$ are independent, SMGD generates a random walk on $\alpha \mathbb{Z}^n$ that is a Markov process.

\section{Error estimates: cost function bounds}\label{convergence_section}

This section presents theorems that control how much the cost function $f$ decreases at each iteration of SMGD.
Our first main theorem, Theorem \ref{first_stoch_theorem}, provides an upper bound on the expected value of $f(x^{t+1})$ in terms of gradient information.
We assume that the gradient of $f$ is $L$-Lipschitz continuous, i.e., $\| \nabla f (x) - \nabla f(y)\| \leq L \| x - y\|.$

\begin{theorem}\label{first_stoch_theorem}
Suppose the cost function $f: \R^n \to \R$ has $L$-Lipschitz gradient $\nabla f$.		
Suppose $G^t$ are independent versions of an unbiased estimator $G$ for $\nabla f$.
Let $\mathcal{E}^t$ denote the event that $\|G^t(x^t)\|_\infty \leq \eta$.
The iterate $x^{t+1}$ of SMGD satisfies
		\begin{equation} \label{first_stoch_theorem_eq}
		\mathbb{E} \left[f(x^{t+1}) \thinspace | \thinspace x^t, \mathcal{E}^t \right] \leq f(x^t)
		+ \frac{L\alpha^2}{2\eta} \mathbb{E} \left[ \|G(x^t) \|_1 \thinspace | \thinspace x^t, \mathcal{E}^t \right] - \frac{\alpha}{\eta}  \|\nabla f(x^t)\|_2^2.
		\end{equation}	
	\end{theorem}
	
The proof of Theorem \ref{first_stoch_theorem} is given in Section \ref{stoch-thm-pf-sec}.   
Section \ref{mini-batch-sec} gives further insight into Theorem \ref{first_stoch_theorem} in the special case of gradient estimators using mini-batches.

	\subsection{Proof of  Theorem \ref{first_stoch_theorem}} \label{stoch-thm-pf-sec}
	
	Gradient-based methods implicitly approximate cost functions by linear surrogates and use this approximation to move towards a minimum.	
	Lipschitz continuity of the gradient is a frequent assumption in SGD literature because it controls the quality of linear approximation. 
	We shall use the following standard lemma, e.g., \cite{Analysis_Book}.

	\begin{lemma}\label{Lemma2}
		Let $f:\mathbb{R}^n\rightarrow \mathbb{R}$ be differentiable and suppose that $\nabla f:\mathbb{R}^n\rightarrow \mathbb{R}^n$ is $L-$Lipschitz.
		 If $D_pf(x)$ denotes the directional derivative of $f$ in the direction $p$ at $x$, then
		 \begin{align}\label{Rn Inequality}
		\left|f(x)+\|p\|_2 D_{{p}/{\|p\|_2}}f(x)-f(x+p)\right|\leq \frac{L\|p\|_2^2}{2}.
		\end{align}
	\end{lemma}

	We use a specific case of Lemma \ref{Lemma2} where $p$ is of a form applicable to SMGD. 
	Recall that the SMGD iterates $x^t$ are defined component-wise by
	\begin{equation} \label{smgd-def-eq}
	x_i^{t+1} = x_i^t - \alpha \ {\rm sgn} ( G^t(x^t)_i) \Delta^t_i.
	\end{equation}
	Since each $\Delta_i \in \{0, 1\}$ is a Bernoulli random variable, let $\Omega^t = \{ i \in \{1, 2, \cdots, n\}: \Delta_i \neq 0\}$ denote the set of indices for which $\Delta_i$ is nonzero.  Namely, $\Omega^t$ contains the indices of the coordinates in which $x^t$ undergoes an update, and \eqref{smgd-def-eq} can be written in vector form as $x^{t+1} = x^t + u^t,$ where
	\begin{equation} \label{ut-def-eq}
	u^t= - \alpha \sum_{i \in \Omega^t}  \text{sgn} \left( G^t(x^t)_i\right) e^i
        \end{equation}
        and $\{ e^i \}_{i=1}^n$ is the canonical basis for $\mathbb{R}^n$.

	\begin{corollary}\label{Key_Cor}
	Let $f:\mathbb{R}^n\rightarrow \mathbb{R}$ be differentiable everywhere and suppose that $\nabla f:\mathbb{R}^n\rightarrow \mathbb{R}^n$ is $L-$Lipschitz.
	The iterates $x^t$ of SMGD satisfy
		 \begin{align} 
	 f(x^{t+1}) \leq f(x^t)+ \frac{L\alpha^2|\Omega^t|}{2}-\alpha\sum_{i\in \Omega^t}\text{sgn}\left(G^t(x^t)_i\right) \frac{\partial f}{\partial x_i}(x^t).\label{key_inequality}
	\end{align} 
	\end{corollary}
	
	\begin{proof}
	Apply Lemma \ref{Lemma2} with $x=x^t$ and $p=u^t$ and note that $\|u^t\|_2^2 = \alpha^2 |\Omega^t|$.	
	\end{proof}

For the proof of Theorem \ref{first_stoch_theorem}, we need two lemmas that compute conditional expectations of the terms in \eqref{key_inequality}.

\begin{lemma}\label{E_Omega}
Let $f: \R^d \to \R$ be a cost function and suppose $G^t$ are unbiased estimators of $\nabla f$.
Let $\mathcal{E}^t$ denote the event that $\|G^t(x^t)\|_\infty \leq \eta$. 
Then SMGD satisfies
\begin{align*}
\mathbb{E}\left[|\Omega^t|\mid x^t, \mathcal{E}^t \right]=\frac{1}{\eta}\mathbb{E}\left[\|G^t(x^t)\|_1 \mid x^t, \mathcal{E}^t \right].\end{align*}
\end{lemma}
	
	\begin{proof}
		Let $\Delta^t_i$ be the Bernoulli random variable with parameter $\frac{1}{\eta}\left|G^t(x^t)_i\right|$, as in the definition of SMGD.
		Observe that $|\Omega^t| = \sum_{i=1}^n \Delta^t_i$, so that by \eqref{strong-indep-assump} we may expand 
		\begin{align*}
		\mathbb{E}\left[|\Omega^t| \ \big| x^t,  \mathcal{E}^t \right]
		=\sum_{i=1}^n\mathbb{E}\left[\Delta^t_i \ \big| \  x^t, \mathcal{E}^t \right]
		&=\sum_{i=1}^n\mathbb{E} \Big[ \mathbb{E}  [\Delta^t_i \mid G^t, x^t] \ \Big| \ x^t, \mathcal{E}^t \Big]\\
		&=\sum_{i=1}^n\mathbb{E}  \Bigg[ \frac{1}{\eta} \left|G^t(x^t)_i\right| \ \Big| \ x^t,  \mathcal{E}^t \Bigg]\\
		&=\frac{1}{\eta}\mathbb{E}  \Bigg[\sum_{i=1}^n  \left|G^t(x^t)_i\right| \ \Big| \ x^t,  \mathcal{E}^t \Bigg]\\
		&=\frac{1}{\eta} \mathbb{E}  \left[ \|G^t(x^t)\|_1 \ \big| \ x^t,  \mathcal{E}^t \right].
		\end{align*}
	\end{proof}

	\begin{lemma}\label{stochastic_sum_lemma}
	Let $f: \R^d \to \R$ be a cost function and suppose $G^t$ are unbiased estimators of $\nabla f$.
	Let $\mathcal{E}^t$ denote the event that $\|G^t(x^t)\|_\infty \leq \eta$.
		Then 
	 \begin{align}
		\mathbb{E}\left[\sum_{i\in \Omega^t} \text{sgn}\left(G^t(x^t)_i\right)\frac{\partial f}{\partial x_i}(x^t) \ \Big| \ x^t, \mathcal{E}^t \right]
		=\frac{1}{\eta}\|\nabla f(x^t)\|_2^2.\label{second_term_stochastic}
		\end{align}
	\end{lemma}
	
	\begin{proof}
	Let $\Delta^t_i$ be the Bernoulli random variable with parameter $\frac{1}{\eta}\left|G^t(x^t)_i\right|$, as in the definition of SMGD.
		Recall that $|\Omega^t| = \sum_{i=1}^n \Delta^t_i$, and compute 
		\begin{align}
		\mathbb{E} \left[\sum_{i\in \Omega^t}\text{sgn}\left(G(x^t)_i\right) \right. & \left.  \frac{\partial f}{\partial x_i}(x^t) \thinspace \Big| \thinspace x^t,  \mathcal{E}^t \right] 
		= \mathbb{E} \left[ \mathbb{E} \left[\sum_{i\in \Omega^t}\text{sgn}\left(G^t(x^t)_i\right)
		\frac{\partial f}{\partial x_i}(x^t) \ \Big| \ G^t, x^t, \mathcal{E}^t \right] \thinspace \Big| \thinspace x^t,  \mathcal{E}^t \right] \notag \\
		&=\mathbb{E} \left[ \mathbb{E} \left[\sum_{i=1}^n\Delta^t_i\cdot\text{sgn}
		\left(G^t(x^t)_i\right) \frac{\partial f}{\partial x_i}(x^t) \ \Big| \ G^t, x^t, \mathcal{E}^t \right] \ \Big| \ x^t,  \mathcal{E}^t \right]\nonumber\\
		&=\mathbb{E} \left[ \sum_{i=1}^n\left(\text{sgn}\left(G^t(x^t)_i\right) 
		\frac{\partial f}{\partial x_i}(x^t)\right)\mathbb{E} [\Delta^t_i\thinspace|\thinspace {G}^t, x^t, \mathcal{E}^t]  \ \Big | \ x^t, \mathcal{E}^t \right] \notag\\
		&=\mathbb{E} \left[ \frac{1}{\eta}\sum_{i=1}^n  G^t(x^t)_i \frac{\partial f}{\partial x_i}(x^t) \ \Big| \ x^t,  \mathcal{E}^t \right] \notag\\
		&=\mathbb{E} \left[ \frac{1}{\eta}\langle G^t(x^t),\nabla f(x^t) \rangle \ \Big| \ x^t,  \mathcal{E}^t \right] \notag \\
		&=  \frac{1}{\eta}\langle \mathbb{E}[G^t(x^t)|x^t,  \mathcal{E}^t],\nabla f(x^t) \rangle  \notag \\
		&=  \frac{1}{\eta}\langle \nabla f(x^t),\nabla f(x^t) \rangle = \frac{1}{\eta} \| \nabla f(x^t) \|^2.	\label{inner_product_form}
		\end{align}
		To reach step \eqref{inner_product_form}, recall that 
		$\mathbb{E} [G^t(x^t)|x^t, \mathcal{E}^t] = \nabla f(x^t)$ by the assumption \eqref{strong-indep-assump}.
	\end{proof}
	
	We are now ready to prove Theorem \ref{first_stoch_theorem}.
	
	\begin{proof}[Proof of Theorem \ref{first_stoch_theorem}]
	We have $\mathbb{E} \left[ \|G^t(x^t)\|_1 \mid x^t, \mathcal{E}^t \right] 
	= \mathbb{E} \left[ \|G(x^t)\|_1 \mid x^t,  \mathcal{E}^t \right]$ since $G^t$ are identically distributed versions of $G$.
	Take conditional expectations on both sides of \eqref{key_inequality} in Corollary \ref{Key_Cor}, and then apply Lemmas \ref{E_Omega} and \ref{stochastic_sum_lemma} to obtain
\begin{align*}
	\mathbb{E}[f(x^{t+1})\mid x^t, \mathcal{E}^t]
	&\leq \mathbb{E}\left[ f(x^t) + \frac{L\alpha^2|\Omega^t|}{2}-\alpha \sum_{i\in \Omega^t}\text{sgn}\left(G^t(x^t)_i\right) \frac{\partial f}{\partial x_i}(x^t) 
	\ \Bigg| \ x^t,  \mathcal{E}^t\right]\\
	&\leq f(x^t) + \frac{L\alpha^2}{2\eta}\mathbb{E} \left[ \|G^t(x^t)\|_1 \mid x^t,  \mathcal{E}^t \right] -\frac{\alpha}{\eta}\|\nabla f(x^t)\|^2\\
	&= f(x^t) + \frac{L\alpha^2}{2\eta}\mathbb{E} \left[ \|G(x^t)\|_1 \mid x^t, \mathcal{E}^t \right] -\frac{\alpha}{\eta}\|\nabla f(x^t)\|^2.
	\end{align*}
\end{proof}

\subsection{Cost function bounds for mini-batch estimators} \label{mini-batch-sec}

Theorem \ref{first_stoch_theorem} depends heavily on the expected $\ell_1$ norm of the 
gradient estimator $\mathbb{E} [ \|G^t(x^t)\|_1 \thinspace | \thinspace x^t, \mathcal{E}^t]$, where $\mathcal{E}^t$ is the event that $\|\nabla G^t(x^t)\|_1 \leq \eta$.
This section studies the quantity $\mathbb{E} [ \|G^t(x^t)\|_1 \thinspace | \thinspace x^t, \mathcal{E}^t]$ 
for the special case when $G^t$ are {\em minibatch gradient estimators}. 
For simplicity, we focus on the case when $\mathcal{E}^t$ occurs almost surely, so that 
$\mathbb{E} [ \|G^t(x^t)\|_1 \thinspace | \thinspace x^t, \mathcal{E}^t] = \mathbb{E} [ \|G^t(x^t)\|_1 \thinspace | \thinspace x^t]$.
Since $G^t$ is independent of $x^t$, we proceed by deriving estimates for $\mathbb{E} [ \|G^t(x)\|_1]$ with fixed $x \in \R^n$.

Mini-batch estimates are a commonly used technique to improve neural network training \cite{AlexNet,Courbariaux, JMLRPaper}.  
This section only considers cost functions of the special form  $f=\frac{1}{m}\sum_{i=1}^mf^i$ where each $f^i$ is differentiable. 
A mini-batch estimator of size $k$ selects $k$ distinct indices $\{i_j\}_{j=1}^k$ uniformly at random from $\{1,2, \cdots, m\}$ and then defines 
$G=\frac{1}{k}\sum_{j=1}^k\nabla f^{i_j}$ as an unbiased estimator of $\nabla f$.  With slight abuse of notation, let $G_k$ denote a minibatch estimator of size $k$.

The following theorem provides bounds on $\mathbb{E} [ \|G^t(x)\|_1]$ for mini-batch estimates.
It will be convenient to introduce some notation for the proof.  Let $[m]$ denote the set $\{1,\dots, m\}$, and let $A^k$ denote the collection of all subsets of size $k$ of a given subset $A \subset [m]$.  For example, $[m]^k$ consists of all subsets of $\{1,\dots, m\}$ containing $k$ elements.  We also let $A^c$ denote the complement of $A$ in $[m]$.

		\begin{theorem}\label{minibatch_result}
		Fix a cost function $f=\frac{1}{m}\sum_{i=1}^mf^i$ where each $f^i$ is differentiable. 
		Let $G_k=\frac{1}{k}\sum_{j=1}^k\nabla f^{i_j}$ be the mini-batch estimator of size $k$ for $\nabla f$.  Let $\| \cdot \|$ be any norm on $\R^n$.  Given $x\in \R^n$, 
		$\mathbb{E}\left[ \|G_k(x)\|\right]$ is non-increasing in $k$ and satisfies the bound
	        \begin{align}
		\mathbb{E}\left[ \|G_k(x)\| \right] \leq\frac{m}{k}\|\nabla f(x)\|+\frac{m-k}{k}\mathbb{E}\left[ \|G_{m-k}(x)\|\right].  \label{minibatch_bound}
		\end{align}
	\end{theorem}

	\begin{proof}
	We first show that $\mathbb{E}\left[ \|G_k(x)\|\right]$ is non-increasing in $k$, by showing that $\mathbb{E}\left[ \|G_k(x)\|\right] \leq \mathbb{E}\left[ \|G_{k-1}(x)\|\right]$.
		By the definition of mini-batch estimates one has 
		\begin{align*}
		\mathbb{E}\left[ \|G_k(x)\|\right] = \frac{1}{\binom{m}{k}}\sum_{A\in [m]^k}\left\|\frac{1}{k}\sum_{i\in A}\nabla f^i(x)\right\|.
		\end{align*} 
		Fix any $A\in [m]^k$ and notice \begin{align*}
		\sum_{i\in A}\nabla f^i(x)=\sum_{B\in A^{k-1}}\sum_{i\in B}\frac{1}{k-1}\nabla f^i(x)
		\end{align*}
		because for each index $i\in A$, there are exactly $k-1$ subsets $B\in A^{k-1}$ containing $i$. Therefore, 
		\begin{align}
		\mathbb{E}\left[\|G_k(x)\|\right]&= \frac{1}{\binom{m}{k}}\sum_{A\in [m]^k}\left\|\frac{1}{k}\sum_{B\in A^{k-1}}\sum_{i\in B}\frac{1}{k-1}\nabla f^i(x)\right\| \notag \\
		&\leq \frac{1}{\binom{m}{k}}\frac{1}{k}\sum_{A\in [m]^k}\sum_{B\in A^{k-1}}\left\|\sum_{i\in B}\frac{1}{k-1}\nabla f^i(x)\right\|. \label{minibatch-pf-eq0}
		\end{align}
		One has that
		\begin{equation} \label{the-sum}
		\sum_{A\in [m]^k}\sum_{B\in A^{k-1}}\left\|\sum_{i\in B}\frac{1}{k-1}\nabla f^i(x)\right\| = (m-k+1)\sum_{B\in [m]^{k-1}}\left\|\sum_{i\in B}\frac{1}{k-1}\nabla f^i(x)\right\|.
		\end{equation}
		To see this, note that the double sum $\sum_{A\in [m]^k}\sum_{B\in A^{k-1}}$ sums over each set $B$ of size $k-1$ once for each size $k$ set $A$ which contains $B$. There are $m-(k-1)$ elements of $[m]$ that can be added to $B$ to get a size $k$ set.  Therefore, each $B$ shows up $m-k+1$ times in this double summation, and \eqref{the-sum} follows.
		
		Combining \eqref{minibatch-pf-eq0} and \eqref{the-sum}, gives
		\begin{align*}
		\mathbb{E}\left[\|G_k(x)\|\right] &\leq \frac{m-k+1}{k\binom{m}{k}}\sum_{B\in [m]^{k-1}}\left\|\sum_{i\in B}\frac{1}{k-1}\nabla f^i(x)\right\|\\
		&=\frac{(m-k+1)\binom{m}{k-1}}{k\binom{m}{k}}\mathbb{E}\left[\|G_{k-1}(x)\|\right].
		\end{align*}
		A computation shows that $\frac{(m-k+1)\binom{m}{k-1}}{k\binom{m}{k}}=1$ and the desired 
		bound $\mathbb{E}\left[\|G_k(x)\|\right] \leq \mathbb{E}\left[\|G_{k-1}(x)\|\right]$ follows.

		It remains to prove the bound \eqref{minibatch_bound}. 
		Using the triangle inequality and $f = \frac{1}{m} \sum_{i=1}^m f^i$ gives
		\begin{align}
		\mathbb{E}\left[ \|G_k\|\right] &=\frac{1}{\binom{m}{k}}\sum_{A\in [m]^k}\left\|\frac{1}{k}\left(\sum_{i\in A}\nabla f^i(x)\right)\right\| \notag \\
		&\leq \frac{1}{\binom{m}{k}}\sum_{A\in [m]^k}\left\|\frac{1}{k}\left(m\nabla f(x)-\sum_{i\in A}\nabla f^i(x)\right)\right\|+\frac{m}{k}\|\nabla f(x)\| \notag \\
		&=\frac{1}{\binom{m}{k}}\sum_{A\in [m]^k}\left\|\frac{1}{k}\sum_{i\in A^c}\nabla f^i(x)\right\|+\frac{m}{k}\|\nabla f(x)\| \notag \\
		&=\frac{m-k}{k}\cdot \frac{1}{\binom{m}{k}}\sum_{A\in [m]^k}\left\|\frac{1}{m-k}\sum_{i\in A^c}\nabla f^i(x)\right\|+\frac{m}{k}\|\nabla f(x)\|.	\label{minibatch-pf-eq1}
		\end{align}
		Since $\binom{m}{k}=\binom{m}{m-k}$ one has
		\begin{align}
		\mathbb{E}\left[ \|G_{m-k}(x)\|\right] &= \frac{1}{\binom{m}{m-k}}\sum_{A\in [m]^{m-k}}\left\|\frac{1}{m-k}\sum_{i\in A}\nabla f^i(x)\right\| \notag \\
		&= \frac{1}{\binom{m}{k}}\sum_{A\in [m]^{k}}\left\|\frac{1}{m-k}\sum_{i\in A^c}\nabla f^i(x)\right\|. \label{minibatch-pf-eq2}
		\end{align}
		Combining \eqref{minibatch-pf-eq1} and \eqref{minibatch-pf-eq2} gives \eqref{minibatch_bound} and completes the proof.
		\end{proof}

		While the above theorem may be difficult to parse, we offer two main insights related to Theorem \ref{minibatch_result}. Recall that the quantity we are bounding controls the performance of SMGD so a smaller value for $\mathbb{E}\left[\|G_k\|\right]$ conceivably implies better algorithm performance. With this in mind, because the expected value is non-increasing in $k$, choosing a larger mini-batch never worsens the performance. Second, as $k$ approaches $m$ the value $\mathbb{E}\left[ \|G_k\|\right]$ approaches $\|\nabla f\|$, the optimal value.

	\section{Error estimates: rates of convergence}\label{stoch_conv_sec}
	
	In this section we analyze the {\em rate of convergence} for SMGD when the cost function is assumed to be strongly convex.  
	A differentiable function $f:\R^n \to \R$ is strongly convex with parameter $\mu$, or simply $\mu-$strongly convex, provided that, for every $x,y\in \mathbb{R}^n$,
		 \begin{align}
		\langle \nabla  f(x)-\nabla  f(y),x-y\rangle \geq \mu \|x-y\|^2_2.  \label{StrongConvDef}
		\end{align}
	It follows from the Cauchy-Schwartz inequality that a $\mu$-strongly convex differentiable function $f$ satisfies
	\begin{align}
	\|\nabla f(x)-\nabla f(y)\|_2 \geq \mu \|x-y\|_2. \label{conv_implication}
	\end{align}
	Strongly convex functions are well-studied in optimization and it is known that a differentiable strongly convex function attains a unique minimum, see e.g., \cite{Convexity_Result}.

	Our next main result
	provides bounds on how fast the iterates of SMGD $x^t$ approach the minimizer $x^*$ of the cost function.
	
	\begin{theorem}\label{iterate_convergence_thm}
	Suppose the cost function $f: \R^n \to \R$ is $\mu$-strongly convex and has $L$-Lipschitz gradient $\nabla f$.		
Suppose $G^t$ are independent versions of an unbiased estimator $G$ for $\nabla f$, and that $G$ is $L$-Lipschitz continuous almost surely.
Let $\mathcal{E}^t$ denote the event that $\|G^t(x^t)\|_\infty \leq \eta$.
The iterates $x^t$ of SMGD satisfy
		 \begin{align}
		\mathbb{E}\left[\|x^{t+1}-x^*\|_2^2 \thinspace | \thinspace x^t, \mathcal{E}^t \right]\leq \left(1-\frac{2\alpha \mu}{\eta}\right)\|x^t-x^*\|_2^2+\frac{L\alpha^2 \sqrt{n}}{\eta} \|x^t-x^*\|_2 + \frac{\alpha^2}{\eta}\mathbb{E}\left[\| G^t(x^*)\|_1 \thinspace | \thinspace x^t, \mathcal{E}^t  \right]. \label{iter-conv-thm-eq}
		\end{align}
	\end{theorem}

\begin{proof}
Let $u^t=-\alpha \cdot \text{sgn}\left(G^t(x^t)\right)\Delta^t_i$ be the random vector defined in \eqref{ut-def-eq}, so that the SMGD iteration may be written as $x^{t+1}=x^t+u^t$.
Thus,
\begin{align}
\mathbb{E}[ \| x^{t+1} - x^*\|_2^2 \thinspace | \thinspace x^t, \mathcal{E}^t] &= \mathbb{E}[ \|x^t - x^* + u^t \|_2^2 \thinspace | \thinspace x^t, \mathcal{E}^t ] \notag \\
&= \|x^t - x^* \|_2^2 + 2 \mathbb{E} [ \langle x^t -x^*, u^t \rangle\thinspace | \thinspace x^t, \mathcal{E}^t ] + \mathbb{E}[ \langle u^t, u^t \rangle \thinspace | \thinspace x^t, \mathcal{E}^t]. \label{conv-pf-eq1}
\end{align}
Note that
\begin{align}
\mathbb{E}[ \langle x^t - x^*, u^t \rangle \thinspace | \thinspace x^t, \mathcal{E}^t] & = \sum_{i=1}^n (x^t-x^*)_i \thinspace \mathbb{E}[u^t_i \thinspace | \thinspace x^t, \mathcal{E}^t] \notag \\
&= -\alpha \sum_{i=1}^n (x^t-x^*)_i \thinspace \mathbb{E}[ {\rm sgn} (G^t(x^t))_i) \Delta^t_i \thinspace |\thinspace x^t, \mathcal{E}^t]. \label{conv-pf-eq2}
\end{align}
Recalling the definition of $\Delta^t$ in \eqref{Delta-def} and using \eqref{strong-indep-assump} gives
\begin{align}
\mathbb{E}[{\rm sgn}(G^t(x^t)_i) \Delta^t_i \thinspace | \thinspace x^t, \mathcal{E}^t]
&= \mathbb{E}[ \mathbb{E}[{\rm sgn}(G^t(x^t)_i) \Delta^t_i \thinspace | \thinspace x^t, \mathcal{E}^t, G^t] \thinspace | \thinspace x^t, \mathcal{E}^t] \notag \\
&= \mathbb{E}[ {\rm sgn}(G^t(x^t)_i) \frac{|G^t(x^t)_i |}{\eta}  \thinspace | \thinspace x^t, \mathcal{E}^t] \notag \\
&= \frac{1}{\eta} \nabla f(x^t). \label{conv-pf-eq3}
\end{align}	
Combining \eqref{conv-pf-eq2} and \eqref{conv-pf-eq3}  gives
\begin{equation} 
\mathbb{E}[ \langle x^t - x^*, u^t \rangle \thinspace | \thinspace x^t, \mathcal{E}^t] = \frac{-\alpha}{\eta} \langle x^t - x^*, \nabla f (x^t) \rangle.  \label{conv-pf-eq4}
\end{equation}
Next note that
\begin{align}
\mathbb{E}[ \langle u^t, u^t \rangle \thinspace | \thinspace x^t, \mathcal{E}^t] 
&=\mathbb{E} [ \mathbb{E}[ \langle R^t, R^t \rangle \thinspace | \thinspace  x^t, G^t, \mathcal{E}^t] \thinspace| \thinspace x^t, \mathcal{E}^t ] \notag \\
&= \alpha^2 \sum_{i=1}^n \mathbb{E} [ \mathbb{E}[(\Delta^t_i)^2 \thinspace |\thinspace  x^t, G^t, \mathcal{E}^t] | x^t, \mathcal{E}^t] \notag \\
&= \alpha^2 \sum_{i=1}^n \mathbb{E} [ \frac{|G^t(x^t)_i|}{\eta}\thinspace  |\thinspace x^t, \mathcal{E}^t] \notag \\
&= \frac{\alpha^2}{\eta} \mathbb{E}[ \| G^t(x^t)\|_1 \thinspace | \thinspace x^t, \mathcal{E}^t].
\end{align}

Combining \eqref{conv-pf-eq1}, \eqref{conv-pf-eq4}, \eqref{conv-pf-eq5}, and using that $\nabla f(x^*)=0$ gives
\begin{align}
\mathbb{E}[ \|x^{t+1} - x^*\|_2^2 \thinspace | \thinspace x^t, \mathcal{E}^t]
&= \|x^t - x^*\|_2^2 - \frac{2\alpha}{\eta} \langle x^t -x^*, \nabla f(x^t) \rangle + \frac{\alpha^2}{\eta} \mathbb{E}[ \| G^t(x^t)\|_1 \thinspace | \thinspace x^t, \mathcal{E}^t] \notag \\
&=  \|x^t - x^*\|_2^2 - \frac{2\alpha}{\eta} \langle x^t -x^*, \nabla f(x^t) -\nabla f(x^*) \rangle + \frac{\alpha^2}{\eta} \mathbb{E}[ \| G^t(x^t)\|_1 \thinspace | \thinspace x^t, \mathcal{E}^t] \notag \\
&\leq \|x^t - x^*\|_2^2 - \frac{2\alpha}{\eta} \langle x^t -x^*, \nabla f(x^t) -\nabla f(x^*) \rangle \notag \\
&+ \frac{\alpha^2}{\eta} \mathbb{E}[ \| G^t(x^t)-G^t(x^*)\|_1 \thinspace | \thinspace x^t, \mathcal{E}^t] 
+ \frac{\alpha^2}{\eta} \mathbb{E}[ \| G^t(x^*)\|_1 \thinspace | \thinspace x^t, \mathcal{E}^t]. \label{conv-pf-eq5}
\end{align}

Applying strong convexity and H\"older's inequality in \eqref{conv-pf-eq5} gives
\begin{align}
\mathbb{E}[ \|x^{t+1} - x^*\|_2^2 \thinspace | \thinspace x^t, \mathcal{E}^t]
&\leq \|x^t - x^*\|_2^2- \frac{2\alpha \mu}{\eta} \|x^t - x^*\|_2^2 \notag \\
&+ \frac{\alpha^2 \sqrt{n}}{\eta} \mathbb{E}[ \| G^t(x^t)-G^t(x^*)\|_1 \thinspace | \thinspace x^t, \mathcal{E}^t] 
+\frac{\alpha^2}{\eta} \mathbb{E}[ \| G^t(x^*)\|_1 \thinspace | \thinspace x^t, \mathcal{E}^t]. \label{conv-pf-eq7}
\end{align}

Finally, since $G$ is $L$-Lipschitz, \eqref{conv-pf-eq7} yields
\begin{align*}
\mathbb{E}[ \|x^{t+1} - x^*\|_2^2 \thinspace | \thinspace x^t]
& \leq \left( 1 - \frac{2 \alpha \mu}{\eta}\right) \|x^t - x^*\|_2^2
+ \frac{\alpha^2 \sqrt{n}L}{\eta} \|x^t - x^*\|_2
+\frac{\alpha^2}{\eta} \mathbb{E}[ \| G^t(x^*)\|_1 \thinspace | \thinspace x^t, \mathcal{E}^t].
\end{align*}
\end{proof}

	Theorem \ref{iterate_convergence_thm} can be viewed as an analogue for SMGD of the convergence results for SGD in \cite{Needell}. 
	Changing notation to match our own, the work in \cite{Needell} shows that, under similar assumptions as Theorem \ref{iterate_convergence_thm}, standard SGD
	with learning rate of $\gamma$ satisfies 
\begin{equation} \label{sgd-nsw-eq}
\mathbb{E} [ \|x^{t+1}-x^*\|_2^2 | x^t] \leq (1-2\gamma \mu)\|x^t-x^*\|_2^2+2\gamma^2L\|x^t-x^*\|_2^2 + 2\gamma ^2 \mathbb{E} [ \|G(x^*)\|_2^2]. 
\end{equation}
This illustrates that the learning rate $\gamma$ for SGD plays an analogous role as the lattice resolution $\alpha$ for SMGD.
It is also worth noting some differences between \eqref{iter-conv-thm-eq} and \eqref{sgd-nsw-eq}.  The middle term in 
\eqref{iter-conv-thm-eq} is a squared norm $\|x^t-x^*\|_2^2$ whereas the middle term in \eqref{sgd-nsw-eq} is not squared; unlike SGD this means that SMGD errors will generally not decrease exponentially fast until saturation.  Moreover, the third terms in \eqref{iter-conv-thm-eq} and \eqref{sgd-nsw-eq} 
reflect the different dependences of SMDG and SGD on the choice of unbiased estimator for $\nabla f$.

	\section{Error bounds: the non-stochastic setting}\label{non_stoch_sec}
	
	In this section we consider the special case of SMGD where the unbiased gradient estimator $G^t$ is the non-stochastic estimate $G=\nabla f$.  
	We shall refer to this special case of SMGD as {\em Markov gradient descent (MGD)}.

The following result is a corollary of Theorem \ref{first_stoch_theorem_eq}.

\begin{corollary} \label{First_non_stoch_theorem}
Suppose the cost function $f: \R^n \to \R$ has $L$-Lipschitz gradient $\nabla f$.	
Let $\mathcal{E}^t$ denote the event $\|\nabla f(x^t)\|_{\infty} \leq \eta$.
The iterate $x^{t+1}$ of MGD satisfies
\begin{align*}
\mathbb{E}\left[f(x^{t+1}) \thinspace | \thinspace x^t, \mathcal{E}^t \right] \leq f(x^t)+ \frac{L\alpha^2}{2\eta} \|\nabla f(x^t)\|_1 - \frac{\alpha}{\eta} \|\nabla f(x^t)\|_2^2.
\end{align*}
\end{corollary}

The following consequence of Corollary \ref{First_non_stoch_theorem} shows that iterates $f(x^{t+1})$ of the cost function decrease in expectation 
when the gradient $\nabla f(x^t)$ has sufficiently large norm.
	
\begin{corollary} \label{nonstoch-cor1}
Suppose the cost function $f: \R^n \to \R$ has $L$-Lipschitz gradient $\nabla f$.	
Let $\mathcal{E}^t$ denote the event $\|\nabla f(x^t)\|_{\infty} \leq \eta$.
If $x\in \R^n$ satisfies $\frac{L\alpha}{2} \|\nabla f(x)\|_1 <  \|\nabla f(x)\|_2^2$,
then the iterate $x^{t+1}$ of MGD satisfies 
\begin{equation} \label{decrease-eq}
\mathbb{E}\left[f(x^{t+1}) \thinspace | \thinspace x^t=x, \mathcal{E}^t \right] < f(x).
\end{equation}
In particular, if $x\in \R^n$ satisfies $\|\nabla f(x)\|_2 > \frac{L\alpha \sqrt{n}}{2}$, then \eqref{decrease-eq} holds.
\end{corollary}

\begin{proof}
It suffices to note that if $\|\nabla f(x)\|_2 > \frac{L\alpha \sqrt{n}}{2}$, then the Cauchy-Schwarz inequality implies 
$$\frac{L\alpha}{2} \|\nabla f(x)\|_1 \leq \frac{L\alpha\sqrt{n}}{2} \|\nabla f(x)\|_2 < \|\nabla f(x)\|^2_2.$$
\end{proof}

The next result gives conditions for expected decrease of the cost function under the assumption of strong convexity. 
\begin{corollary} \label{nonstoch-cor2}
Suppose the cost function $f:\mathbb{R}^n\rightarrow \mathbb{R}$ is $\mu-$strongly convex and has $L-$Lipschitz gradient $\nabla f$. 
Let $x^*$ denote the unique minimizer of $f$. Given a tolerance level $\varepsilon >0$, suppose that
\begin{align}
\alpha < \left( \frac{4\varepsilon \mu}{L^2 n}\right)^{1/2}. \label{nonstoch-cor2-eq1}
\end{align}
Let $\mathcal{E}^t$ denote the event $\|\nabla f(x^t)\|_{\infty} \leq \eta$.
If $x\in \R^n$ satisfies $f(x)-f(x^*)>\varepsilon$, then the iterate $x^{t+1}$ of MGD satisfies 
$$\mathbb{E}\left[f(x^{t+1}) \thinspace | \thinspace x^t=x, \mathcal{E}^t \right]<f(x).$$
\end{corollary}

\begin{proof}
Assume that $x\in \R^n$ satisfies $f(x)-f(x^*)>\varepsilon$.
It suffices to prove that $\|\nabla f(x)\|_2 > \frac{L\alpha \sqrt{n}}{2}$, since the result then follows from Corollary \ref{nonstoch-cor1}.  
We consider two cases depending on whether $\| x - x^*\|$ is large or small.\\
				
\noindent {\em Case 1.}  Suppose that $\|x-x^*\|_2 >\frac{L\alpha \sqrt{n}}{2\mu }$. Applying $\eqref{conv_implication}$ and $\nabla f(x^*)=0$ yields
$$\|\nabla f(x)\|_2 = \| \nabla f(x) - \nabla f(x^*)\|_2 \geq \mu \| x - x^*\|_2 > \frac{L\alpha \sqrt{n}}{2}.$$
		
\noindent {\em Case 2.} Suppose that $\|x-x^*\|_2 \leq \frac{L\alpha \sqrt{n}}{2\mu}$. Define the function $g(r)=f(x^*+ r \frac{x-x^*}{\|x-x^*\|})$, the restriction of $f$ to the line containing both $x^t$ and $x^*$. Observe that $g$ is a strictly convex function of the single variable $r$ with unique minimizer at $r=0$. Moreover, observe that $g'(r)$ is the directional derivative of $f$ at the point $x^*+ru$ in the direction $u=\frac{x-x^*}{\|x-x^*\|_2}$. Because $g$ is convex, we know that this directional derivative is larger than the slope of the secant line of $g$ between $0$ and $r$. Thus, using the Cauchy-Schwarz inequality, we have
		\begin{align}
		\|\nabla f(x)\|_2 \geq \langle \nabla f(x),u\rangle = D_uf(x) > \frac{f(x)-f(x^*)}{\|x-x^*\|}>\frac{2\mu\varepsilon}{L\alpha \sqrt{n}}. \label{conv-cond-pf-eq1}
		\end{align}
		Rewriting \eqref{nonstoch-cor2-eq1} in terms of $\varepsilon$ gives
		\begin{align}
		\varepsilon > \frac{L^2 \alpha^2 n }{4\mu}. \label{conv-cond-pf-eq2}
		\end{align} 
		Combining \eqref{conv-cond-pf-eq1} and  \eqref{conv-cond-pf-eq2} gives $\|\nabla f(x)\|_2 > \frac{L\alpha \sqrt{n}}{4}$.

\end{proof}

The remainder of this section address rates of convergence for MGD.  The next result is a corollary of Theorem \ref{iterate_convergence_thm}, and holds since
$\nabla f(x^*)=0$.
\begin{corollary} \label{MGD_iterate}
	Suppose the cost function $f:\mathbb{R}^n \rightarrow \mathbb{R}$ is $\mu-$strongly convex and has $L-$Lipschitz gradient $\nabla f$.
	Let $\mathcal{E}^t$ denote the event $\|\nabla f(x^t)\|_{\infty} \leq \eta$.
	Let $x^*$ denote the unique minimizer of $f$. 
	The iterate $x^{t+1}$ of MGD satisfies
	\begin{align}
	\mathbb{E}\left[\|x^{t+1}-x^*\|_2^2 \thinspace | \thinspace x^t, \mathcal{E}^t \right]\leq \left(1-\frac{2\alpha \mu}{\eta}\right) \|x^t-x^*\|_2^2 + \frac{L\alpha^2 \sqrt{n}}{\eta}\|x^t-x^*\|_2.
	\end{align}
\end{corollary}

Corollary \ref{MGD_iterate} can be used to provide conditions under which the error $\|x^{t+1} - x^*\|$ for MGD decreases in expectation.
\begin{corollary}\label{Nice_iterate_form}
Suppose the cost function $f:\mathbb{R}^n \rightarrow \mathbb{R}$ is $\mu-$strongly convex and has $L-$Lipschitz gradient $\nabla f$.
Let $\mathcal{E}^t$ denote the event $\|\nabla f(x^t)\|_{\infty} \leq \eta$.
Let $x^*$ denote the unique minimizer of $f$. 
If $x\in \R^n$ satisfies $\|x-x^*\|_2 > \frac{L\alpha \sqrt{n}}{2\mu}$, then the iterate $x^{t+1}$ of MGD satisfies 
	\begin{align*}
	\mathbb{E}\left[\|x^{t+1}-x^*\|^2_2 \thinspace | \thinspace x^t=x, \mathcal{E}^t \right] < \|x-x^*\|^2_2.
	\end{align*}
\end{corollary}

\begin{proof}
	By Corollary \ref{MGD_iterate}, we have
	\begin{align*}
	\mathbb{E}\left[\|x^{t+1}-x^*\|_2^2 \thinspace | \thinspace x^t =x, \mathcal{E}^t \right]
	\leq \|x-x^*\|_2^2 \left(1-\frac{2\alpha \mu}{\eta} + \frac{L\alpha^2 \sqrt{n}}{\eta \|x-x^*\|_2}\right).
	\end{align*}	
	In particular, $\mathbb{E}\left[\|x^{t+1}-x^*\|^2_2 \thinspace | \thinspace  \thinspace x^t=x, \mathcal{E}^t  \right] < \|x-x^*\|^2_2$ holds whenever
	\begin{align}
	\frac{2\alpha \mu}{\eta} > \frac{L\alpha^2 \sqrt{n}}{\eta \|x-x^*\|_2}. \label{cor-cor-eq}
	\end{align}
	Since \eqref{cor-cor-eq} is equivalent to $\|x-x^*\|_2 > \frac{L\alpha \sqrt{n}}{2\mu}$, this completes the proof.
\end{proof}

The following example shows that the conditions $\|\nabla f(x^t)\|>\frac{L\alpha\sqrt{n}}{2}$ and $\|x^t-x^*\|_2 >\frac{L\alpha\sqrt{n}}{2}$ in Corollaries \ref{nonstoch-cor1} and
 \ref{Nice_iterate_form} cannot be weakened.

\begin{example} \label{tight_example}
	Fix a lattice $\alpha \mathbb{Z}^n$. Define the function $f:\mathbb{R}^n\rightarrow \mathbb{R}$ by $f(x_1,\dots , x_n)= \sum_{i=1}^n (x_i-\frac{\alpha}{2})^2$. 
	The unique minimizer of $f$ is $x^*=(\frac{\alpha}{2},\dots, \frac{\alpha}{2})$. 
	Since $\nabla f (x_1,\dots x_n) = (2x_1-\alpha, \cdots, 2x_n - \alpha)$, it follows that $\nabla f$ is $2-$Lipschitz and $f$ is $2-$strongly convex.

	Define $\mathcal{S} = \{ (x_1, \cdots, x_n) \in \alpha \mathbb{Z}^n:  \hbox{ each } x_i \in \{ 0, \alpha\} \}$. Note if $x \in \mathcal{S}$ then $f(x)=\frac{n\alpha^2}{4}$. 
	Further note that if $x^t=0$, then the next iterate of Markov gradient descent satisfies $x^{t+1} \in \mathcal{S}$ because $\frac{\partial f}{\partial x_i}(0)<0$ for all $i$.  
	Therefore, $\mathbb{E}\left[\|x^{t+1}-x^*\|^2_2 \thinspace | \thinspace x^t=0 \right]=\|x^t-x^*\|^2_2$ and 
 $\mathbb{E}\left[f(x^{t+1})-f(x^t) \thinspace | \thinspace x^t=0\right]=0$.  
 This shows that the conclusions of Corollaries \ref{nonstoch-cor1} and \ref{Nice_iterate_form} do not hold.
 However, observe that 
 \begin{align*}
\|\nabla f(0)\|_2 = \left(\sum_{i=1}^n (-\alpha)^2\right)^{1/2} = \alpha \sqrt{n} =\frac{L\alpha \sqrt{n}}{2}
\end{align*}
and
\begin{align*}
\|x^t-x^*\|_2 = \left(\sum_{i=1}^n \frac{\alpha}{2}^2\right)^{1/2} = \frac{\alpha \sqrt{n}}{2}=\frac{L\alpha \sqrt{n}}{2\mu}.
\end{align*}
In particular, the conditions $\|\nabla f(x^t)\|>\frac{L\alpha\sqrt{n}}{2}$ and $\|x^t-x^*\|_2 >\frac{L\alpha\sqrt{n}}{2}$ in Corollaries \ref{nonstoch-cor1} and
 \ref{Nice_iterate_form} are tight.
\end{example}

\section{Experiments and Numerical Validation}\label{numerical_section}

In this section we validate the use of SMGD for training quantized neural networks with three experiments. First, we demonstrate the accuracy of SMGD-trained networks on the standard MNIST and CIFAR-10 datasets. 
Second, we compare SMGD to SGD while holding the amount of memory constant during training. 
Finally, we show the effect that the quality of gradient estimators has on SMGD training by altering minibatch sizes.

\subsection{Performance of SMGD on MNIST and CIFAR-10}
 Our first experiment uses SMGD to train quantized networks with identical architectures as in \cite{Courbariaux}. These experiments validate that SMGD can perform well on some data sets but may not be optimal in other settings. 

We compare 1-bit and 4-bit versions of SMGD for neural network quantization to the performance of the 1-bit BinaryConnect method \cite{Courbariaux} on the MNIST and CIFAR10 datasets.  For the MNIST dataset, we use a feed-forward neural network with 3 hidden layers of 4096 neurons. We use no preprocessing, the ReLU non-linearity, and the softmax output layer. We note that BinaryConnect uses an L2-SVM output layer, batch normalization, and dropout to improve performance while we omit these because the effects of these techniques are not included in our theoretical results. Including these techniques would likely further improve the competitiveness of SMGD.  The first column in Table \ref{table-testerr} shows the test errors for the MNIST dataset.  It is important to emphasize that since SMGD is memory-constrained during training, it is expected that BinaryConnect will outperform SMGD, but the performance of SMGD becomes competitive when more bits are allowed.

On the CIFAR-10 dataset, we use a convolutional architecture which is identical to that in \cite{JMLRPaper}. We observe that while SMGD can perform well on MNIST, it struggles on CIFRAR-10. This could be improved by incorporating advanced techniques such as dropout and SVM output during training, but we suspect that SMGD generally performs worse than other quantization algorithms in this setting.  In particular, we failed to find a good parameter configuration of $\alpha, \eta$ to successfully train a $1$-bit SMGD network on CIFAR-$10$. However, we again emphasize that SMGD weights are quantized during training so a true apples-to-apples comparison does not highlight the usefulness of SMGD. The results of our first set of experiments are summarized in Table \ref{table-testerr}.

\begin{table} 
	\begin{center}
		\begin{tabular}{|c ||c |c|}
			\hline
			\textbf{\hspace{25mm}Method\hspace{25mm}} & \textbf{\hspace{8mm}MNIST\hspace{8mm}} & \textbf{\hspace{8mm}CIFAR-10\hspace{5mm}} \\
			\hline
			Binary Connect & 0.96 & 11.4\\
			SMGD (4-bit) & 1.59 & 27 \\
			SMGD (1-bit) & 6.97 & - \\
			\hline
		\end{tabular}	
		\caption{Test errors of SMGD versus BinaryConnect on MNIST and CIFAR-10.}
		\label{table-testerr}		
	\end{center}
\end{table}

\subsection{Performance of SMGD: memory utilization during training}
Our second experiment highlights the motivation for using SMGD: the network is compressed during training as well as at run time. This is in contrast to the existing techniques that we are aware of which require full precision during training.  Moreover, many other neural network quantization techniques, e.g., \cite{Courbariaux}, require more memory during training than a full precision network trained with SGD. To study this issue, we compare a quantized network trained with SMGD and a full precision network trained with SGD where the memory during training is held approximately constant.

Training a network requires the storage of the weights and intermediate neural outputs as well as computation and storage of partial derivatives. The weights and partial derivatives take up an overwhelming amount of this memory, so let us compute how much savings SMGD provides in this area. SMGD requires $q$ bits per weight and $2$ bits to store each partial derivative after quantization. Computing the partial derivatives takes an additional $32$ bits per weight when we use mini-batches as we must aggregate the full-precision gradient over many input signals before quantization. However, in the online setting where we process only one image at a time, we can compute the partial derivatives one-by-one. So, in the setting without mini batches we require only $2+q$ bits-per-weight to train our network. When we use mini batches this number is $32+q$.

 Full-precision networks, on the other hand, require full-precision for weights and partial derivatives leading to $64$ bits-per-weight. We recall that other quantization methods typically require more memory because they store both auxilliary and quantized weights. Therefore, other methods generally require at least $\frac{64}{2+q}$ times more memory during online training than an SMGD network. Therefore, for a fixed amount of memory, one can use a network that is approximately $\frac{64}{2+q}$ times larger than the full-precision networks which allows for better accuracy in a memory-constrained environment.

The details of our second experiment are as follows. First, we trained a full-precision neural network with a batch size of $1$ on the MNIST data set to determine a baseline performance. Then, we compute the size of the SMGD-trained network that requires the same amount of memory and train that network for the same number of epochs as the full-precision network. The results of these experiments for $q=4,5,6$ bit quantization are shown in Figure \ref{fig:online_comp}.  While not included in the figure, the result for $q=3$ bits is still favorable, but   the results degrade for $q=2$ and $q=1$ bit networks on these small architectures.  

\begin{figure}
	\centering
	\includegraphics[width=\linewidth]{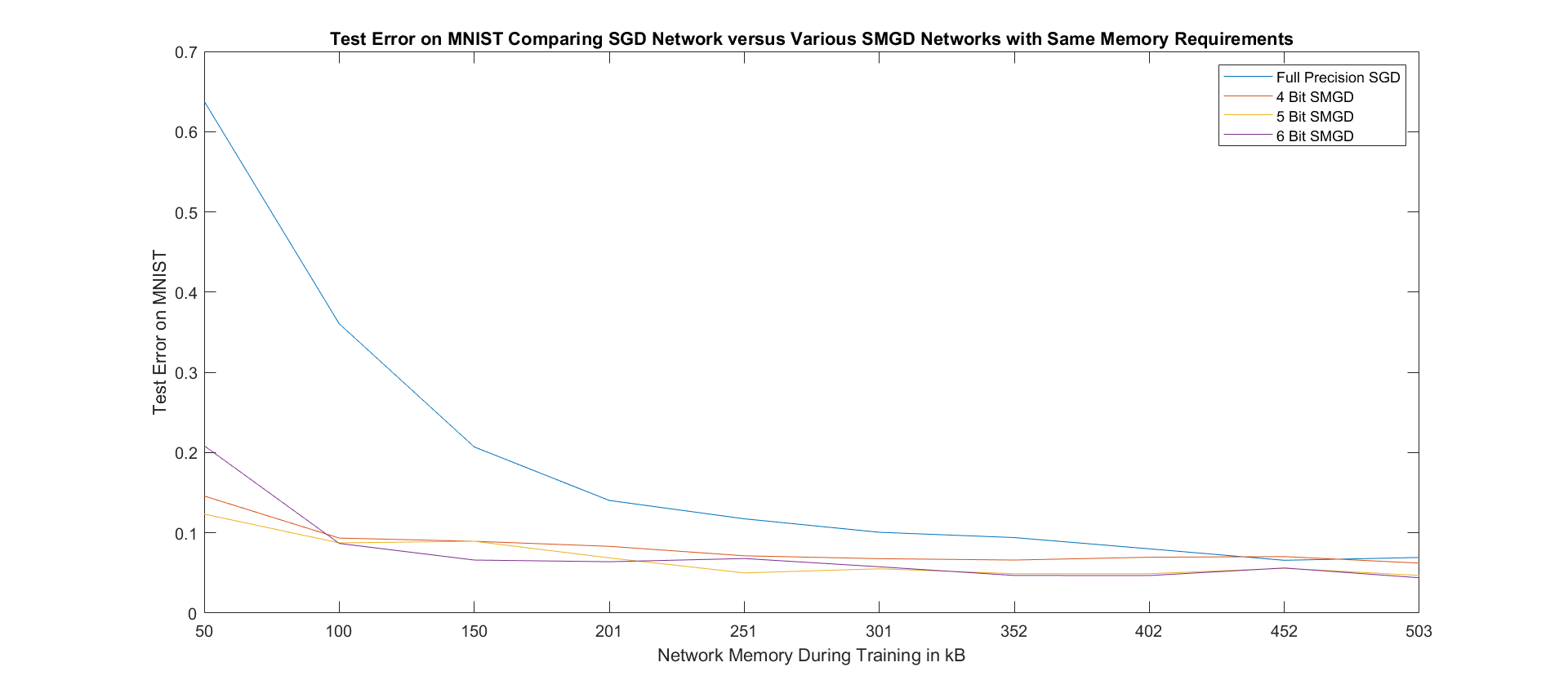}
	\caption{Comparison of test accuracy of various training methods each using approximately the same memory to store the weights.}
	\label{fig:online_comp}
\end{figure}

The motivation for SMGD is consistent with the fact that training neural networks is not a {one size fits all} problem.  The choice of training method should be dependent on the setting in which the learning occurs. We offer that SMGD may be best implemented in the `memory-contrained during training' environment while other quantization methods are better in other constrained settings. 
Table \ref{table-constraints} itemizes some recommendations regarding best training practices under various constraints on the resulting network and training process.
%Below we itemize some recommendations we have gathered from the available literature regarding best training practices under various constraints on the resulting network and training process
%\begin{itemize}
%	\item Without any constraints: SGD, AdaGrad, Adam, etc.
%	\item Unconstrained during training, memory constrained at run-time: BinaryConnect, QNN, etc.
%	\item Time constrained during both run and test-time: XNOR, QNN, etc.
%	\item Memory constrained during training: SMGD
%\end{itemize}

\begin{table}
	\begin{center}
		\begin{tabular}{|c |c|} 
			\hline
			\textbf{\hspace{25mm}Constraints \hspace{25mm}} & \textbf{\hspace{8mm} Methods\hspace{8mm}} \\
			\hline
			\hline
			None &  SGD, AdaGrad \cite{adagrad}, Adam \cite{adam} \\ \hline
			Unconstrained during training; memory constrained at run-time & BinaryConnect \cite{Courbariaux}, QNN \cite{JMLRPaper} \\ \hline
			Time constrained during both run and test-time & XNOR \cite{xnor}, QNN \cite{JMLRPaper} \\ \hline
			Memory constrained during training & SMGD \\			\hline
		\end{tabular}
		\caption{Network training methods that are suitable under different constraints.}
		\label{table-constraints}	
	\end{center}
\end{table}

\subsection{Effect of minibatch size on SMGD}
Our final experiment highlights the effect of increased minibatch size and illustrates the improvements suggested by Theorem \ref{minibatch_result} together with
Theorem \ref{first_stoch_theorem}. 
We trained a network using SMGD and with increasing mini-batch sizes. The experiment illustrates that as mini-batch size increases SMGD achieves better training error until it {saturates}. 
Moreover, we see that while increasing the mini-batch size improves the performance of SMGD, there are diminishing returns as the batch size grows. 
The results of this experiment are contained in Figure \ref{fig:effectofvariance}.  

\begin{figure}
	\centering
	\includegraphics[width=\linewidth]{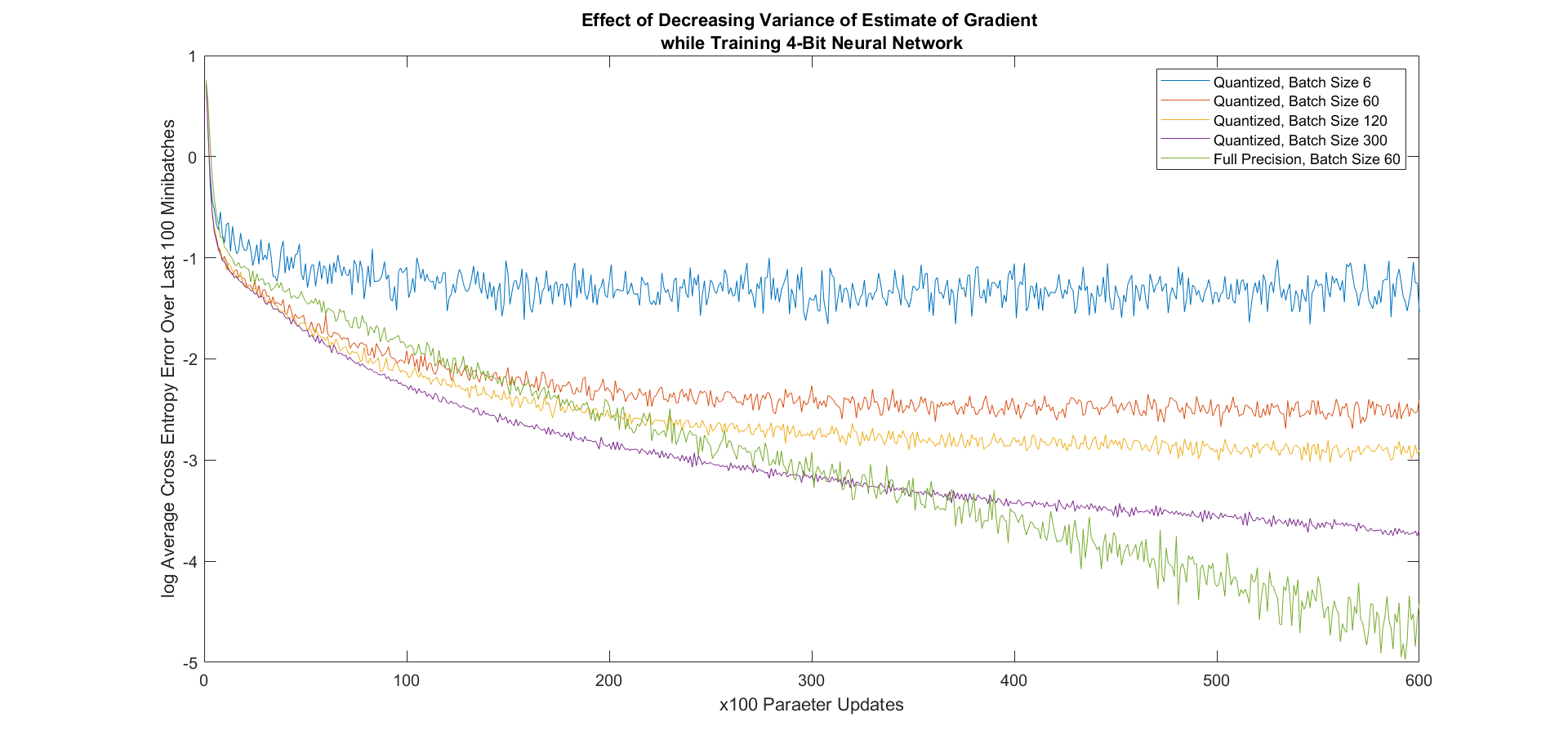}
	\caption{Training errors for various batch sizes trained with SMGD on identical network architectures}
	\label{fig:effectofvariance}
\end{figure}

\section{Conclusion}
This paper introduces stochastic Markov gradient descent (SMGD), a method for training low-bit neural networks in the setting where memory is constrained during training. 
We established theoretical guarantees for SMGD and have shown its viability through numerical experiments.  Open directions of work include extending SMGD to conjugate gradient methods, incorporating more advanced ideas such as batch normalization and adaptive gradients, and relaxing our algorithm to allow for finer quantization of the gradient.

%Each of these ideas may admit similar if not identical analysis and may or may not enjoy increased performance, so we did not include them in the present work. 
%While our emphasis was on using SMGD to train quantized neural networks, the theoretical guarantees ensure that SMGD may be applied in various other circumstances. 
%For example, the same ideas were used in \cite{Shen} for digital half-toning. 
%Therefore, SMGD should be viewed as a discrete optimization method that happens to be particularly useful as a tool in the large box of neural network training algorithms.
%the training process of determining a good choice of weight parameters $w$ poses nontrivial practical challenges...

\section*{Acknowledgments}
The authors thank Weilin Li, Gal Mishne, Jeff Sieracki, and Penghang Yin for helpful conversations.
A.~Powell was supported in part by NSF DMS Grant 1521749, and gratefully acknowledges the Academia Sinica Institute of Mathematics (Taipei, Taiwan) for its hospitality and support.

%The authors thank Gal Mishne for valuable discussions.
% Thank your friend I met at Maryland?
%A.~Powell was supported in part by NSF DMS Grant 1521749, and gratefully acknowledges the Academia Sinica Institute of Mathematics (Taipei, Taiwan) for their hospitality and support.

\bibliographystyle{siam}

\begin{thebibliography}{10}
	
	\bibitem{Courbariaux}
	{\sc M.~Courbariaux, Y.~Bengio, and J.-P. David}, {\em Binary{C}onnect:
		Training deep neural networks with binary weights during propagations}, in
	Advances in Neural Information Processing Systems 28, C.~Cortes, N.~D.
	Lawrence, D.~D. Lee, M.~Sugiyama, and R.~Garnett, eds., Curran Associates,
	Inc., 2015, pp.~3123--3131.
	
	\bibitem{Cybenko}
	{\sc G.~Cybenko}, {\em Approximation by superpositions of a sigmoidal
		function}, Mathematics of Control, Signals and Systems, 2 (1989),
	pp.~303--314.
	
	\bibitem{Daubechies2019}
	{\sc I.~Daubechies, R.~DeVore, S.~Foucart, B.~Hanin, and G.~Petrova}, {\em
		Nonlinear approximation and deep (relu) networks}, arxiv:1905.02199 (2019).
	
	\bibitem{adagrad}
	{\sc J.~Duchi, E.~Hazan, and Y.~Singer}, {\em Adaptive subgradient methods for
		online learning and stochastic optimization}, Journal of Machine Learning
	Research, 12 (2011), pp.~2121--2159.
	
	\bibitem{Kmeans}
	{\sc Y.~Gong, L.~Liu, M.~Yang, and L.~D. Bourdev}, {\em Compressing deep
		convolutional networks using vector quantization}, arxiv:1412.6115 (2014).
	
	\bibitem{NN_Book}
	{\sc I.~Goodfellow, Y.~Bengio, and A.~Courville}, {\em Deep Learning}, MIT
	Press, 2016.
	
	\bibitem{LSTM}
	{\sc S.~Hochreiter and J.~Schmidhuber}, {\em Long short-term memory}, Neural
	Computation, 9 (1997), pp.~1735--1780.
	
	\bibitem{LossAware}
	{\sc L.~Hou, Q.~Yao, and J.~T. Kwok}, {\em Loss-aware binarization of deep
		networks}, arxiv:1611.01600 (2017).
	
	\bibitem{JMLRPaper}
	{\sc I.~Hubara, M.~Courbariaux, D.~Soudry, R.~El-Yaniv, and Y.~Bengio}, {\em
		Quantized neural networks: Training neural networks with low precision
		weights and activations}, Journal of Machine Learning Research, 18 (2018),
	pp.~1--30.
	
	\bibitem{adam}
	{\sc D.~P. Kingma and J.~Ba}, {\em Adam: A method for stochastic optimization},
	arxiv:1412.6980 (2014).
	
	\bibitem{AlexNet}
	{\sc A.~Krizhevsky, I.~Sutskever, and G.~E. Hinton}, {\em Imagenet
		classification with deep convolutional neural networks}, in Advances in
	Neural Information Processing Systems 25, F.~Pereira, C.~J.~C. Burges,
	L.~Bottou, and K.~Q. Weinberger, eds., Curran Associates, Inc., 2012,
	pp.~1097--1105.
	
	\bibitem{Lecun}
	{\sc Y.~LeCun, Y.~Bengio, and G.~Hinton}, {\em Deep learning}, Nature Cell
	Biology, 521 (2015), pp.~436--444.
	
	\bibitem{Convexity_Result}
	{\sc E.~Moulines and F.~R. Bach}, {\em Non-asymptotic analysis of stochastic
		approximation algorithms for machine learning}, in Advances in Neural
	Information Processing Systems 24, J.~Shawe-Taylor, R.~S. Zemel, P.~L.
	Bartlett, F.~Pereira, and K.~Q. Weinberger, eds., Curran Associates, Inc.,
	2011, pp.~451--459.
	
	\bibitem{Needell}
	{\sc D.~{Needell}, N.~{Srebro}, and R.~{Ward}}, {\em {Stochastic Gradient
			Descent, Weighted Sampling, and the Randomized Kaczmarz algorithm}},
	arxiv:1310.5715 (2013).
	
	\bibitem{Analysis_Book}
	{\sc Y.~Nesterov}, {\em Introductory Lectures on Convex Optimization: A Basic
		Course}, Springer Publishing Company, Incorporated, 1~ed., 2014.
	
	\bibitem{xnor}
	{\sc M.~Rastegari, V.~Ordonez, J.~Redmon, and A.~Farhadi}, {\em Xnor-net:
		Imagenet classification using binary convolutional neural networks},
	arxiv:1603.05279 (2016).
	
	\bibitem{Shen}
	{\sc J.~J. Shen}, {\em Least-squares halftoning via human vision system and
		markov gradient descent (ls-mgd): Algorithm and analysis}, SIAM Review, 51
	(2009), pp.~567--589.
	
	\bibitem{zouweishen}
	{\sc Z.~Shen, H.~Yang, and Z.~S.}, {\em Deep network approximation
		characterized by number of neurons}, arxiv, arxiv:1906.05497 (2020).
	
	\bibitem{Penghang}
	{\sc P.~Yin, S.~Zhang, J.~Lyu, S.~Osher, Y.~Qi, and J.~Xin}, {\em Blended
		coarse gradient descent for full quantization of deep neural networks},
	Research in the Mathematical Sciences, 6 (2019).
	
	\bibitem{Zhou2}
	{\sc D.-X. Zhou}, {\em Deep distributed convolutional neural networks:
		universality}, Analysis and Applications, 16 (2018), pp.~895--919.
	
	\bibitem{Zhou1}
	\leavevmode\vrule height 2pt depth -1.6pt width 23pt, {\em Universality of deep
		convolutional neural networks}, Applied and Computational Harmonic Analysis,
	48 (2020), pp.~787--794.
	
	\bibitem{DoReFa}
	{\sc S.~Zhou, Z.~Ni, X.~Zhou, H.~Wen, Y.~Wu, and Y.~Zou}, {\em Dorefa-net:
		Training low bitwidth convolutional neural networks with low bitwidth
		gradients}, arxiv:1606.06160 (2016).
	
\end{thebibliography}

\end{document}